\documentclass[12pt, a4paper]{article}
\usepackage{fullpage}
\usepackage[english]{babel}

\usepackage[utf8]{inputenc}
\usepackage{xspace}
\usepackage{amsmath,amsthm,amssymb,mathtools}
\usepackage{lmodern}
\usepackage{bbm}
\usepackage{url}
\usepackage[caption=false]{subfig}
\usepackage{natbib}
\usepackage{algorithm}
\usepackage[noend]{algpseudocode}
\usepackage{xcolor}
\usepackage{tikz}
\usepackage{graphicx}
\renewcommand{\labelenumi}{(\alph{enumi})}
\renewcommand\theenumi\labelenumi

\DeclareMathOperator{\Prob}{Pr}
\newcommand{\prob}[1]{\mathord{\Pr}\mathord{\left(#1\right)}}

\newcommand{\expect}[1]{\mathord{E}\mathord{\left(#1\right)}}

\newcommand{\W}{W}
\newcommand{\lambert}{Lambert~$W$\xspace}

\newcommand{\ie}{i.\,e.\xspace}
\newcommand{\eg}{e.\,g.\xspace}

\newtheorem{theorem}{Theorem}
\newtheorem{lemma}[theorem]{Lemma}

\newtheorem{corollary}[theorem]{Corollary}

\newenvironment{proofof}[1]{\begin{proof}[Proof of~#1]}{\end{proof}}

\newcommand{\tend}{t_{\mathrm{end}}}
\newcommand{\wmax}{w_{\max}}
\newcommand{\wmin}{w_{\min}}
\newcommand{\smax}{s_{\max}}
\newcommand{\smin}{s_{\min}}

\newcommand{\om}{\textsc{OneMax}\xspace}
\newcommand{\onemax}{\om}

\newcommand{\oea}{\mbox{${(1 + 1)}$~EA}\xspace}

\newcommand{\ooea}{\oea}

\newcommand{\valley}{\textsc{Valley}\xspace}

\newcommand{\R}{\ensuremath{\mathbb{R}}}

\newcommand{\N}{\ensuremath{\mathbb{N}}}
\newcommand{\Z}{\ensuremath{\mathbb{Z}}}

\newcommand{\eps}{\varepsilon}

\newcommand{\tbase}{T_{base}}
\newcommand{\gammatimebase}{4.21 \gamma mn \ln (2m^2/\delta)}
\newcommand{\timebase}{4.21 mn \ln (2m^2/\delta)}
\newcommand{\minimuma}{\ln (4(\ell-1)/\delta)}

\let\originalleft\left
\let\originalright\right
\renewcommand{\left}{\mathopen{}\mathclose\bgroup\originalleft}
\renewcommand{\right}{\aftergroup\egroup\originalright}

\author{
Benjamin Doerr\\
Laboratoire d'Informatique (LIX)\\ CNRS, \'Ecole Polytechnique\\ Institut Polytechnique de Paris\\ Palaiseau, France \\
\and
  Amirhossein Rajabi\\
  DTU Compute\\
  Technical University of Denmark \\
	Kgs. Lyngby, Denmark \\
  \and
    Carsten Witt\\
    DTU Compute\\
  Technical University of Denmark \\
	Kgs. Lyngby, Denmark
}

\title{Simulated Annealing is a\\Polynomial-Time Approximation Scheme for the \\Minimum Spanning Tree Problem}

\begin{document}
\maketitle

{\sloppy
\begin{abstract}
 We prove that Simulated Annealing with an appropriate cooling schedule computes arbitrarily tight constant-factor approximations to the minimum spanning tree problem in polynomial time. This result was conjectured by Wegener (2005). More precisely, denoting by $n, m, \wmax$, and $\wmin$ the number of vertices and edges as well as the maximum and minimum edge weight of the MST instance, we prove that simulated annealing with initial temperature $T_0 \ge \wmax$ and multiplicative cooling schedule with factor $1-1/\ell$, where $\ell = \omega  (mn\ln(m))$, with probability at least $1-1/m$  computes in time $O(\ell (\ln\ln (\ell) + \ln(T_0/\wmin) ))$ a spanning tree with weight at most $1+\kappa$ times the optimum weight, where $1+\kappa =  \frac{(1+o(1))\ln(\ell m)}{\ln(\ell) -\ln (mn\ln (m))}$. Consequently, for any  $\eps>0$, we can choose $\ell$ in such a way that a $(1+\eps)$-approximation is found in time $O((mn\ln(n))^{1+1/\eps+o(1)}(\ln\ln n + \ln(T_0/\wmin)))$ with probability at least~$1-1/m$.  In the special case of so-called $(1+\eps)$-separated weights, this algorithm computes an optimal solution (again in time $O( (mn\ln(n))^{1+1/\eps+o(1)}(\ln\ln n + \ln(T_0/\wmin)))$), which is a significant speed-up over Wegener's runtime guarantee of  $O(m^{8 + 8/\eps})$.
\end{abstract}

\section{Introduction}

The theory of randomized search heuristics, mostly in the last 25 years, has considerably increased our understanding of this class of algorithms. A closer look at this field shows that in the early years, significant efforts were devoted also to simulated annealing~(SA) \citep{SasakiH88,JerrumS98,Wegener05,JansenW07}, whereas more recently these algorithms at most appear in side results of works focused on other heuristics. Due to this decline in attention, the gap between theory and practice, at least as wide in heuristics as in classic algorithms, is even wider for~SA. 

Since we do not see a reducing interest in SA in practice~\citep{FranzinS19}, with this first theoretical work solely devoted to SA after a longer time, we aim at reviving the theoretical analysis of this famous  heuristic. To this aim, we revisit a classic problem, namely how SA computes minimum spanning trees (MSTs)~\citep{Wegener05}. We are, of course, not finally interested in using SA for this purpose -- for this several very efficient near-linear time algorithms are known --, but we use this problem to try to understand the working principles of SA.

Wegener's seminal work~\citep{Wegener05} is well-known for the construction of an instance of the MST problem where the Metropolis algorithm with any fixed temperature fails badly, but SA with a simple multiplicative cooling schedule computes an optimal solution efficiently. Much less known, but equally interesting is another result in this work, namely that SA with a suitable multiplicative cooling schedule can efficiently find optimal solutions to the MST problem when the edge weights are $(1+\eps)$-separated. 

\begin{theorem}[\cite{Wegener05}]\label{thm:wegener}
  Let $G = (V, E)$ with $w : E \to \Z_{>0}$ be an instance of the MST problem. Let $\eps > 0$ be such that for all edges $e_1, e_2 \in E$, we have that $w(e_1) > w(e_2)$ implies $w(e_1) \ge (1+\eps) w(e_2)$. Assume further that $w(e) \le 2^m$ for all $e \in E$. Then SA with initial temperature $T_0 = 2^m$ and cooling factor $\beta = (1+\eps/2)^{-m^{-7-8/\eps}}$ with probability $1 - O( 1/m)$ finds an optimal solution in at most $2 \log_2(1+\eps/2)^{-1} m^{8 + 8/\eps}$ iterations.
\end{theorem}

Wegener~\citep{Wegener05} conjectured that his SA algorithm for general weights instead of $(1+\eps)$-separated ones computes $(1+\eps)$-approximate minimum spanning trees, that is, trees with weight at most $(1+\eps)$ times the weight of a true minimum spanning tree. While this conjecture is very natural, it was never proven.

Our main result is that Wegener's conjecture is indeed true, even though our proof does not confirm his statement that ``it is easy to generalize our result to prove that SA is always highly successful if one is interested in $(1 + \eps)$-optimal spanning trees.'' More precisely, we show the following result (see Theorem~\ref{theo:approx-depend-ell} for a slightly stronger, but more complicated version of this result). We note that SA cannot compute $(1+\eps)$-approximations for sub-constant $\eps$, see again~\citep{Wegener05}, so in this sense our result is as good as possible.

{\addtolength\leftmargini{-5ex}
\begin{quote}
\itshape    Let $\eps>0$ be a constant. Consider a run of SA with cooling factor $\beta = 1 - 1/\ell$, where $\ell = (mn \ln (m))^{1+ 1/\eps+o(1)}$, 
and $T_0 \ge \wmax$ on an instance of the MST problem. Then there is a time $T^*=O((mn\ln(n))^{1+1/\eps+o(1)}(\ln \ln n+\ln(T_0/\wmin)))$
 such that with probability at least $1-1/m$, at all times $t \ge T^*$ the current solution is a $(1+\eps)$\nobreakdash-approximation.
\end{quote}}

Due to the use of proof methods not available at that time, our time bound is significantly better than Wegener's. To compute a $(1+\eps)$-approximation, or to compute an optimal solution when the edge weights are $(1+\eps)$-separated (see Theorem~\ref{theo:sa-eps-separated}), our runtime guarantee is roughly $O((mn\log n)^{1 + 1/\eps} \log \frac{\wmax}{\wmin})$ as opposed to $O(m^{8 + 8/\eps})$ in Theorem~\ref{thm:wegener}.

Mostly because of a different organization of the proof, our result gives more insights into the influence of the algorithm parameters. Our result only applies to initial temperatures $T_0$ that are at least the maximum edge weight. This is very natural since with substantially smaller temperatures, the heaviest edge cannot be included in the solution with reasonable probability (this follows right from the definition of the algorithm). It is also not difficult to prove that once the temperature is somewhat below the smallest edge weight, then no new edges will ever enter the solution (see Lemma~\ref{lem:heavy-edges} for the precise statement of this result). This implies that there is no reason to run the algorithm longer than roughly for time $\log_{1/\beta}(T_0 / \wmin) = O(\ell \log(T_0 / \wmin))$, see Theorem~\ref{theo:sa-approximation} for the details. From the perspective of the algorithm user, this is an interesting insight since it gives an easy termination criterion. Also without understanding the precise influence of the cooling factor~$\beta$ on the approximation quality, this insight motivates to use the algorithm for decreasing values of $\beta$, say $\beta_i = 2^{-i}$, always until the above-determined time is reached, and follow this procedure until a sufficiently good MST approximation is found.

The remainder of this paper is organized as follows. In Section~\ref{sec:previous}, we describe the most relevant previous works. We define SA and the minimum spanning tree problem in Section~\ref{sec:prelims}. The core of this work is our mathematical runtime analysis in Section~\ref{sec:runtime}. 
Afterwards, in Section~\ref{sec:sa-eps-separated}, we give the result carried out for the MST problem with $(1+\eps)$-separated weights. 
The paper ends with a conclusion and a discussion of possible future works.

\section{Previous Work}\label{sec:previous}

As mentioned in the introduction, there are relatively few runtime analyses for SA as discrete optimization algorithm, see also the survey~\cite{Jansen11}.

The first such result~\cite{SasakiH88} proves that SA can compute good approximations to the maximum matching problem. A closer look at the result reveals that a constant temperature is used, that is, the SA algorithm is in fact the special case of the Metropolis algorithm. It has to be noted that to obtain a particular approximation quality, the temperature has to be set suitably. In this light, the following result from~\cite{GielW03} shows a light advantage for evolutionary algorithms: When running the \oea with standard mutation rate on this problem, then the expected first time to find a $(1+\eps)$\nobreakdash-approximation is $O(m^{2\lceil 1/\eps \rceil})$. Note that in this result, the parameters of the algorithm do not need to be adjusted to the desired approximation rate.

For a different problem, namely the bisection problem, it was shown in~\cite{JerrumS98} that SA, again with constant temperature, can solve certain random instances in quadratic time. 

Wegener's above mentioned work~\citep{Wegener05} on the MST problem was the first to show that for some non-artificial problem, a non-trivial cooling schedule is necessary. 

A runtime analysis of the Metropolis algorithm on the classic benchmark \onemax was conducted in~\cite{JansenW07}. Not surprisingly, the ability to accept inferior solutions is not helpful when optimizing this unimodal function. The interesting side of this result, though, is that the Metropolis algorithm is efficient on \onemax only for very small temperatures of asymptotic order $O(\log(n)/n)$. 

A recent study~\citep{WangZD21} on the deceiving-leading-blocks (DLB) problem shows that here the Metropolis algorithm with a constant temperature has a good performance, beating the known runtime results for evolutionary algorithms by a factor of $\Theta(n)$. We note that the DLB problem, just as the MST problem, has many local optima which all can be left by flipping two bits. 

As side results of a fundamental analysis of hyper-heuristics, two easy lower bounds on the runtime of the Metropolis algorithm (that is, SA with constant temperature) are proven in~\cite{LissovoiOW19}: (i)~The Metropolis algorithm needs time $\tilde \Omega(n^{d-1/2})$ on cliff functions with constant cliff width~$d$ and super-polynomial time when the cliff width is super-polynomial. (ii)~The Metropolis algorithm with a temperature small enough to allow efficient hill-climbing needs exponential time to optimize jump functions. 

As part of a broader analysis of single-trajectory search heuristics, it was found that the Metropolis algorithm can optimize all weakly monotonic pseudo-Boolean functions in at most exponential time~\cite{Doerr21tcsUB}.

Some more results exist on problems designed for demonstrating a particular phenomenon. In~\cite{DrosteJW00}, a problem called \valley is designed that has the property the Metropolis algorithm with any temperature needs at least exponential expected time, whereas SA with a suitable cooling schedule only needs time $O(n^5 \log n)$. In~\cite{JansenW07}, examples are constructed where one of \oea and SA has a small polynomial runtime and the other has an exponential runtime. Also, a class of functions is constructed where both algorithms have a similar performance despite dealing with the local optimum in a very different manner. In~\cite{OlivetoPHST18}, a class of problems with tunable width and depths of a valley of low fitness is proposed. It is proven that the performance of the elitist \oea is mostly influenced by the width of the valley, whereas the performance of the Metropolis algorithm and a similar non-elitist algorithm inspired from population genetics is mostly influenced by the depths of the valley.

For evolutionary algorithms, for which the theory is more developed than for SA, there are a larger number of  results showing that they can serve as approximation algorithms 
for  optimization problems, including NP-hard problems  \citep{NeumannW10}. However, 
results describing an approximation scheme where the user can provide a parameter~$\eps$ to the 
evolutionary algorithm to compute a $(1+\eps)$-approximation are rare; apart 
from the maximum matching problem mentioned above, we are only aware of related results for parallel (1+1)~EAs, (1+1)~EAs with ageing and 
simple artificial immune systems on the number partitioning problem \citep{WittSTACS05,CorusOYAIJ19} and for 
an evolutionary algorithm on the multi-objective shortest path problem \citep{HorobaECJ10}. Evolutionary algorithms 
that approximate the optimum are also known in the subfield of fixed-parameter tractability. 
While most of these results prove an approximation within a constant factor or growing slowly with the problem dimension, 
there are 
also statements similar to approximation schemes for the vertex cover problem \citep{NeumannSuttonFPTBookChapter2020}. However,
in general 
it is safe to say that there are only few 
results in the literature that characterize 
very simple randomized search heuristics like the   \oea and SA as  
polynomial-time approximation schemes for 
classical (non-noisy) combinatorial optimization problems.

Finally, we remark that the classical \ooea and a variant of randomized local search can solve the MST problem in expected pseudo-polynomial time $O(m^2 \log(n\wmax))$ \citep{NeumannW07}. While SA in general does not solve the problem in expected polynomial time, its time bound to achieve a $(1+\epsilon)$-approximation (see  Theorem~\ref{theo:approx-depend-ell} below)
can be smaller than the time bound for the \ooea in certain cases where $m=\omega(n)$ and $\epsilon$ is a constant. 

\section{Preliminaries}\label{sec:prelims}

We now define the SA algorithm and the MST problem. Also, we state a technical tool our main proof builds on. 

\begin{algorithm}[t]
	\caption{Simulated Annealing (SA) with starting temperature  $T_0$ and cooling factor $\beta\le 1$ for the minimization of $f\colon\{0,1\}^n\to \R$}
	\label{alg:sa}
	\begin{algorithmic}
	\State Select $x^{(0)}$  from $\{0, 1\}^n$.
		\For{$t \gets 0, 1, \dots$}
		\State Create $y$ by flipping a bit of~$x^{(t)}$ chosen uniformly at random.
		\If{$f(y) \le f(x^{(t)})$}
		\State $x^{(t+1)} \gets y$.
		\Else{}
		\State $x^{(t+1)} \gets y$ with probability $e^{(f(x^{(t)})-f(y))/T_t}$ and 
		\State $x^{(t+1)} \gets x^{(t)}$ otherwise.
		\EndIf
		\State $T_{t+1}\coloneqq T_t \cdot \beta$.
		\EndFor
	\end{algorithmic}
\end{algorithm}

\emph{Simulated annealing (SA)} is a simple stochastic hill-climber first proposed as optimization algorithm in~\cite{KirkpatrickGV83}. Different from a true hill-climber it may, with small probability, also accept inferior solutions. Working with bit-string representations, we use the classic \emph{bit-flip neighborhoods}, that is, the neighbors of a solution are all other solutions that differ from it in a single bit value. For the acceptance of inferior solutions, we use the widely accepted \emph{Metropolis condition}, that is, a solution with fitness loss $\delta$ over the current solution is accepted with probability $e^{-\delta/T}$, where $T$ is the current \emph{temperature}. The temperature is usually not taken as constant, but is reduced during the run of the algorithm. This allows the algorithm to accept worsening moves easy in the early stages of the run, whereas later worsening moves are accepted with smaller probability, bringing the algorithm closer to a true hill-climber. The choice of the \emph{cooling schedule} is a critical decision in the design of a SA algorithm. A popular choice, already proposed in~\cite{KirkpatrickGV83}, is a multiplicative cooling schedule (also called geometric cooling scheme). Here we start with a given temperature~$T_0$ and reduce the temperature by some factor $\beta$ in each iteration. This common variant of SA, see Algorithm~\ref{alg:sa} for the pseudocode, was regarded also in the predecessor work of Wegener~\citep{Wegener05}.

The \emph{minimum spanning tree (MST)
problem} is defined as follows. We are given an undirected, connected, weighted graph $G=(V,E)$. We denote by~$n$ its number of vertices and by~$m$ 
its number of edges. Let the set of edges be $E = \{e_1, \dots, e_m\}$. The weight of edge $e_i$, where $i\in\{1,\dots,m\}$,
is a positive number~$w_i$. We write $\wmin \coloneqq \min\{w_i \mid i\in\{1,\dots,m\}\}$ and 
$\wmax \coloneqq \max\{w_i \mid i\in\{1,\dots,m\}\}$ for the minimum and maximum edge weight.

The task in the MST problem is to find a subset $E'\subseteq E$ such that $(V,E')$ is a spanning tree of~$G$ 
having minimal total weight $w(E') = \sum_{e_i \in E'} w_i$. We use the natural bit-string representation for sets $E'$ of edges, that is, a bit string $x = (x_1,\dots,x_m) \in \{0,1\}^m$ represents the set $E(x) = \{e_i \mid x_i = 1\}$. As objective function, we use the sum of the weights of the selected edges when these form a connected graph on $V$ and $\infty$ otherwise:  
\[
f(x) = \begin{cases}
w_1x_1+\dots+w_ mx_m & \text{if $(V,E(x))$ is connected,}\\
\infty & \text{otherwise.}
\end{cases}
\]
Here $\infty$ can be replaced by an extremely
large value without essentially changing 
the result. To ensure that we start with a feasible solution (one that has finite objective value), we assume that SA is initialized with the all-ones 
string $x^{(0)}=(1,\dots,1)$. From this initial string, SA can move to 
solutions having fewer edges by flipping one-bits; however, it will never 
accept solutions that are not connected due to their 
infinitely high $f$-value. We note that, similarly 
to the analysis of the \oea on the MST 
problem \citep{NeumannW07}, one could use a more involved fitness
function to penalize connected components 
and thus lead the algorithm towards connected 
subgraphs when the current solution is not connected. However, since we assumme SA to start from a connected solution 
and connected solutions will not be replaced 
with disconnected solutions with the present definition 
of~$f$, this would not provide 
new insights. Overall, our setup is the same as the one used by Wegener~\citep{Wegener05}.

When the temperature has become sufficiently low, 
it is likely that SA has reached a solution describing 
a spanning tree. If this spanning tree is suboptimal, 
improvements require a change of at least~$2$ bits. 
Since SA only flips one bit per iteration, this is only possible by temporarily including one more edge, \ie, 
closing a cycle, and then removing another edge from 
the cycle in the next iteration. This requires a temperature still being sufficiently high for the 
temporary inclusion to be accepted.

Our measure of complexity is the first hitting time~$T^*$  
for a certain set of solutions~$S^*$, \eg, globally optimal solutions 
or solutions satisfying a certain approximation guarantee with respect to the set 
of global optima. That 
is, we give bounds on the smallest~$t$ such that SA has 
found a solution in~$S^*$. Due to the probabilistic nature of 
the algorithm, we will usually 
give bounds that hold with high probability, 
\eg, with probability~$1-1/n$. The expected value of~$T^*$ 
may be undefined since the cooling schedule may make it less 
and less likely to hit the set~$S^*$ when the algorithm has 
been unsuccessful during the steps where a promising temperature held. This is different from the analysis of, \eg, simple 
evolutionary algorithms, where one often considers  the so-called \emph{runtime} 
as the first hitting time of the set of optimal solutions 
and bounds the expected runtime. However, as described in detail by 
Wegener~\citep{Wegener05}, there are simple restart schemes for~SA that 
guarantee expected polynomial optimization times if there 
is a sufficiently high probability of a single run being successful in polynomial time.

The proof of our main result uses multiplicative drift analysis 
as state-of-the-art technical tool, which was not available to Wegener~\citep{Wegener05}. 
The  multiplicative drift theorem in Theorem~\ref{theo:multdrift-upper} below 
goes back to  
\cite{DoerrJW12algo} and was enhanced with tail bounds
in  \cite{DoerrG13algo}. We give a 
 slightly generalized presentation that can be found in~\cite{LehreW21}.

\begin{theorem} [Multiplicative Drift, cf. \citep{DoerrJW12algo,DoerrG13algo,LehreW21}] 
\label{theo:multdrift-upper}
Let $(X_t)_{t\ge 0}$, be a stochastic process, adapted to a filtration~$\mathcal{F}_t$, 
over a state space $S\subseteq \{0\}\cup [\smin,\smax]$, where 
$\smin>0$ and $\{0\}\in S$. Suppose that there exists a $\delta>0$ such that
for all $t\ge 0$, we have
\[
\expect{X_t-X_{t+1}\mid \mathcal{F}_t}\ge \delta X_t.
\]
Then the first hitting time 
$T:=\min\{t\mid X_t=0\}$ satisfies 
\[
\expect{T\mid \mathcal{F}_0} \le \frac{\ln(X_0/\smin)+1}{\delta}.
\]
Moreover, 
$\Prob(T> (\ln(X_0/\smin)+r)/\delta) \le e^{-r} $ for any $r>0$.
\end{theorem}

\section{SA as Approximation Scheme for the Minimum Spanning Tree Problem}\label{sec:runtime}

In this section, we prove our main results on how well SA computes approximate solutions for the MST problem. These results easily imply improved bounds for the previously regarded special case of $(1+\eps)$-separated instances, see Section~\ref{sec:sa-eps-separated}.

\subsection{Main Results and Proof Outline}

As outlined above in the introduction, 
this paper revisits Wegener's~\citep{Wegener05} analysis of SA 
on the MST problem. Our main result is Theorem~\ref{theo:eps-from-ell} below, proving 
that SA is a polynomial-time approximation scheme 
for the MST problem as originally conjectured by 
Wegener. The statement of our main theorem 
describes the approximation quality and the required 
time to reach it as a function of the cooling 
factor, the desired 
success probability and of course the 
instance parameters. 
Theorem~\ref{theo:approx-depend-ell} takes the dual 
perspective of computing cooling schedules and 
running times that allow SA to find a $(1+\eps)$-approximation for a given $\eps$ 
with high probability. 

We now present the main theorem and a  
variant of it, 
corresponding to the two perspectives mentioned 
above for analyzing the approximation quality.

\begin{theorem} 
\label{theo:eps-from-ell}
Let $\delta < 1$.
Consider a run of SA with multiplicative cooling schedule with $\beta = 1 - 1/\ell$ for some $\ell = \omega(mn\ln (m/\delta))$ and $T_0 \ge \wmax$ on an instance of the MST problem. With probability at least $1-\delta$, at all times $t \ge (\ell/2) \ln \left(\frac{\minimuma T_0}{\wmin}\right)$ the current solution is a $(1+\kappa)$\nobreakdash-approximation, where
    \[1+\kappa \le (1+o(1))\frac{\ln(\ell/\delta)}{\ln (\ell) - \ln (mn\ln (m/\delta))} .\]
\end{theorem}

\begin{theorem}
\label{theo:approx-depend-ell}
Let $\delta=\omega(1/(mn\ln n))$ and $\delta < 1$, $\eps>0$.
Consider a run of SA with $\beta = 1 - 1/\ell$ for $\ell = (mn\ln(m/\delta))^{1+1/\eps}$ and $T_0 \ge \wmax$ on an instance of the MST problem. With probability at least~$1-\delta$, at all times $t \ge (\ell/2) \ln \left(\frac{\minimuma T_0}{\wmin}\right)$ the current solution is a $(1+o(1))(1+\eps)$\nobreakdash-approximation.
\end{theorem}

The last theorem is stated in somewhat 
weaker, but 
simpler form in the following corollary. 
In particular, it gives a concrete time bound 
until SA has computed a $(1+\eps)$\nobreakdash-approximation with 
probability at least~$1-\delta$, where 
$\delta$ and $\eps$ are chosen by the 
user.

\begin{corollary}
\label{cor:final-corollary}
Let $\eps>0$ be a constant and $\delta=\omega(1/(mn\ln n))$. Consider a run of SA with $\beta = 1 - 1/\ell$, where 
\[\ell = (mn\ln (m/\delta))^{1+ 1/\eps+o(1)},\] 
and $T_0 \ge \wmax$ on an instance of the MST problem. With probability at least $1-\delta$, at all times 
$t \ge T^*\coloneqq (\ell/2) \ln \left(\frac{\minimuma T_0}{\wmin}\right)$ the current solution is a $(1+\eps)$\nobreakdash-approximation. Moreover, 
\[T^*=O((mn\ln(n))^{1+1/\eps+o(1)}(\ln \ln n+\ln(T_0/\wmin))).\]
\end{corollary}

The idea of the proof of all 
results formulated above is to consider phases in the optimization process, concentrating on different intervals for the edge weights, 
with the size and center of the intervals decreasing over time. In each  
 phase, the number of edges chosen from such an interval will achieve some close-to-optimal value with high probability. After the end of the phase, 
 the temperature of SA is so low that basically no more changes occur to the 
edges with weights in the interval.

In more detail, the proofs of Theorem~\ref{theo:eps-from-ell} and its variant 
are composed of several lemmas. We are now going to 
outline the main ideas of these lemmas and how 
they relate to each other in the roadmap of the final proof.

It is useful to formulate the main results in terms of a cooling factor $\beta=1-1/\ell$ for some $\ell>1$ since
$\ell$ carries the intuition of a ``half-life'' for 
the temperature; more precisely, after $\ell$ iterations of SA the temperature has decreased by the constant 
factor of $(1-1/\ell)^{\ell} \approx e^{-1}$.  
Lemma~\ref{lem:heavy-edges} is (on top of the usual 
graph parameters and the starting temperature) based on $\ell$, a weight~$w$ and some parameter~$a$. Intuitively, it describes a point of time~$t_w$ after which 
edges of 
weight at least~$w$ are no longer flipped in with 
high probability and can be ignored 
for the rest of the analysis due to an exponential 
decay in the probability of accepting search 
points of higher $f$-value. This probability depends on the 
parameter~$a$ which will be optimized later in the 
composition of the main proof.

While Lemma~\ref{lem:heavy-edges} will be used to show 
that edges above a certain weight are no longer included
in the current solution after the temperature has 
dropped sufficiently, Lemma~\ref{lem:cheap-connected-components}, which 
is the main lemma in our analysis, deals with the 
structure of the current solution after edges of a
certain weight~$w$ are no longer included. It considers   
connected components
that can be spanned by cheaper edges and 
states that these connected components are essentially connected 
in an optimal way in the whole solution up to 
multiplicative deviations of a factor $(1+\kappa)$ 
in the weights of the connecting edges. 
Lemma~\ref{lem:cheap-connected-components} uses 
 careful edge exchange arguments in its proof 
 and bounds the time to do these exchanges in  
 a multiplicative drift analysis. Moreover, it features 
 another parameter called~$\gamma$ that will be 
 optimized later along with the above-mentioned~$a$. 

Lemma~\ref{lem:bijection} puts together the previous 
two lemmas to consider the run of SA over 
up to $n$ phases depending on the weight spectrum of the graph 
until the temperature has dropped to a value being 
so small 
that no more changes are accepted. This will 
be the final solution considered in the main proof. 
Essentially, having listed the weights of an 
MST decreasingly, the lemma will match the weights 
of the final solution to the weights of the MST and 
show for each element in the list that the final 
solution matches the weight of the element up to 
a factor $1+\kappa$. Its proof uses a bijection 
argument proved by induction to apply Lemma~\ref{lem:cheap-connected-components} and 
is crucially different from Wegener's analysis.

The final lemma, Lemma~\ref{lem:choice-of-parameters}, finds choices for the parameter~$\gamma$ to minimize the bound~$1+\kappa$ on the approximation ratio. Its proof uses several results 
from calculus. Afterwards,  Theorem~\ref{theo:eps-from-ell} also chooses the 
parameter~$a$ carefully and arrives at the first 
statement on the approximation ratio depending 
on $\ell$, the desired success probability~$1-\delta$, 
and the graph parameters, only. The second main theorem, 
Theorem~\ref{theo:approx-depend-ell} then 
essentially translates parameters into each 
other to compute~$\ell$ and to express time bounds 
based on the desired~$\eps$. A weaker but 
simpler formulation of that theorem is finally stated
in Corollary~\ref{cor:final-corollary}.

\subsection{Detailed Technical Analysis}
In this subsection, we collect the technical lemmas  and theorems 
outlined above.

Let $a>1$ and $t_w$ be the earliest point of time when $T(t_w)\le w/a$.
In the following lemma, we state that the probability that SA accepts edges of weight~$w$ after~$t_w$ is exponentially small with respect to~$a$. It shows that after the temperature becomes less than~$w$, the probability of accepting such an edge is sharply decreasing.

\begin{lemma} \label{lem:heavy-edges}
Consider a run of SA with multiplicative cooling schedule with $\beta = 1-1/\ell$ and $T_0 \ge \wmax$ on an instance of the MST problem.
Let $\ell>2$, $1<a\le \ell-1$ and for any $w>0$, $t_w$ be the earliest point of time when $T(t_w)\le w/ a$. It holds that no new edge of weight at least~$w$ is included in the solutions after time~$t_w$ with probability at least
\[1-\frac{2(\ell-1)}{ae^a},\]
which is at least $1-\delta/2$ for $\delta<1$, if we set $a\ge \minimuma$.
\end{lemma}

\begin{proof}
Let $s$ be an edge of weight at least~$w$, which is not in the solution at the beginning of the step~$t_w$. Let $t\in \N_{\ge 0}$ and $E^{(t_w+t)}_s$ be the event of accepting the edge~$s$ at step~$t_w+t$. 
This event happens if the edge $s$ is flipped with probability~$1/m$ and the algorithm accepts this worse solution.
Thus
\begin{align*}
    \prob{E^{(t_w)}_s} &= m^{-1} \cdot \exp\left(\frac{-w}{T(t_w)}\right) \le  e^{-a}/m.
\end{align*}
For all integers $t\ge 0$, we have $T(t_w+t)=T(t_w)(1-1/\ell)^t$. Then
\begin{align*}
    \prob{E^{(t_w+t)}_s} &= m^{-1}  \exp\left(\frac{-w}{T(t_w)(1-1/\ell)^t}\right) \\
    & \le m^{-1}e^{-a(1+\frac 1{\ell-1})^{t}} \\
    &\le m^{-1}e^{-a(1+\frac t{\ell-1})},
\end{align*}
where we used the inequality $(1+x)^r\ge1+rx$ for $x>-1$ and $r\in \N_{\ge0}$.

Let $E^{\ge t_w}_{s}$ be the event of accepting the edge~$e$ of weight at least~$w$ after step~$t_w$ at least once. Then, using the geometric series sum formula, we get
\begin{align*} 
    \prob{E^{\ge t_w}_{s}} & \le \sum^{\infty}_{t=0}\prob{E^{(t_w+t)}_e} \le \sum^{\infty}_{t=0}  m^{-1}e^{-a(1+\frac t{\ell-1})}  \\ 
    &= m^{-1}\frac{e^{-a}}{1-e^{-a/(\ell-1)}} \le m^{-1}\frac{e^{-a}}{1-(1-\frac{ a}{2(\ell-1)})} \\
    &= m^{-1}\frac{2(\ell-1)}{ae^a} \label{eq:t_w},
\end{align*}
where we have $a\le \ell-1$ and  use the inequality $e^{-x}\le 1-x/2$ for $0\le x\le1$.

Since there are $m$ edges, with probability
$1-\frac{2(\ell-1)}{ae^{a}}$,
there is no inclusion of edges after their corresponding steps~$t_w$.

Moreover, if we set $a\ge \minimuma$, the probability is at least
\[1-\frac{2(\ell-1)}{\minimuma \cdot 4(\ell-1)/\delta}=1-\frac{\delta/2}{\minimuma}\ge 1-\delta/2,\]
where we have $\ell>2$ and $\delta<1$.
\end{proof}

In the following lemma, we consider a time interval of length $\gammatimebase+1$ starting from~$t_w$ (for fixed~$a$) and prove that at the end of this period,
there are no edges of weight at least~$w$ left that could be replaced by an edge of weight at most~$w/(1+\kappa)$, where $\kappa$ depends on the algorithm parameter $\ell$ and parameters $\gamma$ and $a$. We optimize these parameters later in this paper.

\begin{lemma} \label{lem:cheap-connected-components}
Let $\gamma>1$, $\delta<1$, $\ell>2$, $a>1$.
Consider a run of SA with multiplicative cooling schedule with $\beta = 1-1/\ell$ and $T_0 \ge \wmax$ on an instance of the MST problem.
    Let $t_w$ be the earliest point of time when $T(t_w)\le w/a$, and assume that
    no further edges of weight at least~$w$ are added to the solution from time~$t_w$.
    Let
    \[1+\kappa=\frac{a\exp\left(\gamma \frac{\timebase}{\ell-1}\right)}{\ln \gamma}.\]
    
    Let $n_w$ be the number of connected components in the subgraph using only edges with weight at most $w/(1+\kappa)$ in $G$. After time~$t_w+\gammatimebase$, the number of edges in the current solution with weight at least~$w$ is at most $n_w-1$ with probability at least~$1-\delta/(2m)$. 
\end{lemma}

\begin{proof} Let $\tbase=\timebase$.
     We analyze the steps $t_w,\dots,t_w+\gamma \tbase$. The temperature during this phase is at least 
     \[T(t_w)(1-1/\ell)^{\gamma \tbase}\ge T(t_w)e^{-(\gamma \tbase)/(\ell-1)},\]
     so
     the probability to accept a chosen edge with weight at most~$w/(1+\kappa)$ in one step is bounded from below by 
 \[\exp \left( \frac{-w/(1+\kappa)}{T(t_w)e^{-\gamma \tbase/(\ell-1)}}\right) = \exp\left(-\frac{ae^{\gamma \tbase/(\ell-1)}}{(1+\kappa)}\right)
 = 1/\gamma \]
 during this phase. By our assumption in the statement, we do not include edges of weight at least~$w$.

Let us partition the set of edges with weight at least~$w$ in the current solution~$x$, that is, the graph $G_x=(V,E(x))$, into three disjoint subsets. An edge $e=\{u,v\}$ with weight at least~$w$ has one of the following three properties,
\begin{enumerate}
    \item the edge~$e$ lies on a cycle in $G_x$;
    \item the edge~$e$ does not lie on a cycle, but there is at least one edge~$e'\in E\setminus E(x)$ with weight at most~$w/(1+\kappa)$ such that~$e$ lies on a cycle in the graph $(V,E(x) \cup \{e'\})$;
    \item the edge~$e$ has neither of the two properties. In this case, we call this edge essential for the current and forthcoming solutions.
\end{enumerate}

As long as an edge with weight at least $w$ is not essential, it can either be removed from the current solution or become an essential edge. When the edge disappears, since its weight is at least~$w$, it will not appear again. 

Also, when the edge becomes essential, it remains essential in the solution to the end, because in order to create a cycle containing this edge, an edge with weight at least~$w$ has to appear, which does not happen, and also removing this edge makes the graph unconnected. 

We claim that the number of essential edges does not exceed~$n_w-1$. In order to prove this, we define the graph $H=(V_H,E_H)$ as follows. There is a vertex in $V_H$ for each connected component of the induced subgraph on the edges of weight at most~$w/(1+\kappa)$ in~$G$, and there is an edge between two vertices $v_i,v_j \in V_H$ if there is an essential edge~$e=\{u,v\}$ in the solution that $u$ and $v$ belong to the corresponding connected components~$C_i$ and~$C_j$ respectively. Formally, let $C=\{C_1,\dots,C_{n_w}\}$ be the connected components of the induced subgraph on the edges of weight at most~$w/(1+\kappa)$. Then, $V_H=\{v_1,\dots,v_{n_w}\}$ and 
\[E_H=\left\{\{i,j\} \mid \exists \text{ essential } e=\{u,v\}, u\in C_i, v\in C_j \right\}.\] 

We claim that there is no essential edge with both endpoints in the same $C_i$. To prove this, we assume for contradiction that there is such an edge $e=\{u,v\}$. Then, since $e$ is essential, it cannot be on a cycle in the current solution. Let $S_u$ and $S_v$ denote the sets of vertices connected to $u$ and $v$ respectively using edges in the solution but~$e$. $S_u\cup S_v=V(G)$ because the solution is always connected. Since $e$ is essential, there is no edge with weight at most $w/(1+\kappa)$ in $G$ from $S_u$ to $S_v$ (see the property (2)), so there is no such cheap edges in $G$ from $S_u \cap C_i$ to $S_v\cap C_i$, which results in that there is a partition of vertices of $C_i$ that are disconnected in the subgraph  using only edges with weight at most $w/(1+\kappa)$ in~$G$, which contradicts the definition of $C_i$.
Also, $H$ has to be a forest since we also know that essential edges are not on a cycle. Therefore, since there are $n_w$ connected components, there are at most $n_w-1$ essential edges. 

Now, in the next paragraphs, we state the number of steps needed to remove edges with weight at least~$w$ or to make them essential.  We consider some epochs consisting of $2m$ iterations each and let $X_t$ be the random variable denoting the number of non-essential edges with weight at least $w$ whose exclusion is possible at epoch~$t$. We claim that 
\[\Delta_t(s) \coloneqq \expect{X_t-X_{t+1} \mid X_t=s}\ge s\cdot  (1-e^{-3})n^{-1}/(2\gamma).\]
If no cycle with a non-essential edge~$e=\{u,v\}$ with weight at least $w$ exists, the probability of creating such a cycle by adding the cheap edge considered in Case~2 between $S_u$ and $S_v$ in each step is at least $1/(\gamma m)$ and in $m$ steps, is at least~\[1-\left(1-\frac 1{\gamma m}\right)^m\ge 1-e^{-1/\gamma}\ge 1/(2\gamma),\]
where we have $1+x \le e^x$ for all $x\in R$ and the inequality $e^{-x}\le 1-x/2$ for $0\le x\le1$.

Then, after the cycle is created in the first $m$ iterations, or the cycle  already existed, the probability of the exclusion of such an edge in $m$ steps of the second half of the epoch is only $(1-e^{-3})n^{-1}$ because the probability of observing at least one edge from the cycle of length~$k$ in $m$ steps is $1-(1-k/m)^m\ge1-(1-3/m)^m\ge 1-e^{-3}$, and the probability that the edge selected is $e$ equals $1/n$. Altogether, the probability of excluding a non-essential edge with weight at least~$w$ is at least $(1-e^{-3})n^{-1}/(2\gamma)$, which results in decreasing $X_t$ by at least one because removing $e$ might also make some other edges essential. Since there are $s$ non-essential edges, we have $\Delta_t(s)\ge s\cdot (1-e^{-3})n^{-1}/(2\gamma)$. Since there can be at most~$m$ essential edges at the beginning, we have $X_0 \le m$.
Assume $Y$ denotes the number of epochs needed to have only essential edges with weight at least~$w$.  Using the upper tail bound of multiplicative drift in Theorem~\ref{theo:multdrift-upper}, we have
\begin{align*}
  \prob{Y>\frac{\ln (2m/\delta)+\ln X_0}{(1-e^{-3})n^{-1}/(2\gamma)}}\le e^{-\ln (2m/\delta)}=\delta/(2m).
\end{align*}

Since each epoch consists of $2m$ iterations, 
\[2m \cdot 2 (1-e^{-3})^{-1}n \gamma\ln (2m^2/\delta) \le \timebase\]
is sufficient to arrive at a solution where all edges of weight at least~$w$ are essential.
\end{proof}

SA does with high probability not accept an inclusion of any edge  using Lemma~\ref{lem:heavy-edges} when the temperature is colder than~$\wmin/a$ for some~$a$ that is 
still a parameter chosen later. This is the time from when the solution is invariant. Let $t_{\wmin}$ be the earliest time when $T(\wmin)\le \wmin/a$ and $\tend \coloneqq t_{\wmin}$. 

In the following lemma, we show that there is a bijective relation between the edges of the solution at time~$\tend$ and a MST such that the ratio between the weights of corresponding edges is less than~$(1+\kappa)$.

\newcommand{\Tmst}{\mathcal{T}}

\begin{lemma} \label{lem:bijection}
Let $\delta < 1$, $\gamma>1$, $\ell=\omega(1)$ and $a\ge \minimuma$. Let 
    \[1+\kappa=\frac{a\exp\left(\gamma \frac{\timebase}{\ell-1}\right)}{\ln \gamma}.\]
Consider a run of SA with multiplicative cooling schedule with $\beta = 1-1/\ell$ and $T_0 \ge \wmax$ on an instance of the MST problem. Assume that $\Tmst^*$ is a minimum spanning tree and $\Tmst'$ is the solution of SA at time~$\tend$ where $T(\tend)\le \wmin/a$. 

For an arbitrary spanning tree $\Tmst$, let $w_\Tmst=(w_\Tmst(1),\dots,w_\Tmst(n-1))$ be a decreasingly sorted list of the weights on its edges, 
\ie, $w_\Tmst(j) \ge w_\Tmst(i)$ for all $1 \le j \le i \le n-1$. With probability at least~$1-\delta$, we have
\begin{align*} \label{claim:bijection}
w_{\Tmst^*}(k) \le w_{\Tmst'}(k) < (1+\kappa) w_{\Tmst^*}(k)  \text{ for each }k\in [1..n-1].
\end{align*}
\end{lemma}

\begin{proof}
We recall that $t_w$ is the earliest point of time when $T(t_w)\le w/a$. With probability~$1-\delta/2$, edges of weight~$w$ are not included after their corresponding times~$t_w$ via Lemma~\ref{lem:heavy-edges}. Thus conditional on this event, we can use Lemma~\ref{lem:cheap-connected-components} stating that with probability at least~$1-\delta/(2m)$, the number of edges with weight at least $w$ is at most~$n_w-1$. This condition must hold for at most $m$ distinct values, happening with probability at least~$1-\delta/2$ according to a union bound. Altogether, since the event in Lemma~\ref{lem:heavy-edges} must happen with probability $1-\delta/2$ and the condition in Lemma~\ref{lem:cheap-connected-components} must hold for all weights, with probability at least~$1-\delta$, the statement in Lemma~\ref{lem:cheap-connected-components} is valid for all possible weights.

We use induction on the index~$k$. The case $k=0$ is trivial as the basic step. Regarding the inductive step, assume that for all $0\le k \le i-1$, the inequality is valid. If $i=n$, the claim is proved. Otherwise, let $w_{\Tmst^*}(i)$ be the next unique largest weight and $j$ be the largest index that $w_{\Tmst^*}(j)=w_{\Tmst^*}(i)$. In fact, we have \[w_{\Tmst^*}(i-1) < w_{\Tmst^*}(i)=\dots=w_{\Tmst^*}(j) < w_{\Tmst^*}(j+1).\]
There are exactly $j-i+1$ edges with weight $w_{\Tmst^*}(i)$ in the minimum spanning tree $\Tmst^*$.  The number of connected components in $G$ using only edges at most $w_{\Tmst^*}(i)$ is $i$ since they are connected using~$i-1$ edges in $\Tmst^*$.
Using Lemma~\ref{lem:cheap-connected-components} with $w=(1+\kappa)w_{\Tmst^*}(i)$ and considering $n_w=i$, there are at most $i-1$ edges with weight at least $(1+\kappa)w_{\Tmst^*}(i)$ in $\Tmst'$, which means that the rest of the weight values in $\Tmst'$ are less than $(1+\kappa)w_{\Tmst^*}(i)$.
Since we know that the graph cannot be connected using less than $j$ edges with weight at least~$w_{\Tmst^*}(i)$, we can conclude that there are at least $j$ edges with weight between $w_{\Tmst^*}(i)$ and $(1+\kappa)w_{\Tmst^*}(i)$. Therefore, for $i\le k \le j$, the inequality suggested above holds.
\end{proof}

With the above lemmas at hand, we can prove the first theorem. Given $\ell$, Theorem~\ref{theo:sa-approximation} states the approximation ratio that the algorithm with cooling schedule $\beta=1-1/\ell$ can obtain.

\begin{theorem} \label{theo:sa-approximation}
Let $\delta < 1$, $\gamma>1$ and $\ell=\omega(1)$.
Consider a run of SA with multiplicative cooling schedule with $\beta = 1-1/\ell$ and $T_0 \ge \wmax$ on an instance of the MST problem. For $a\ge \minimuma$, with probability at least~$1-\delta$, at all times $t \ge (\ell/2) \ln \left(a T_0/\wmin\right)$ the current solution is a $(1+\kappa)$\nobreakdash-approximation where
    \[1+\kappa=\frac{a\exp\left(\gamma \frac{\timebase}{\ell-1}\right)}{\ln \gamma}.\]
\end{theorem}
\begin{proof}
We consider the time $\tend$ when $T(\tend) \le \wmin/a$ and show the approximation result for the current solution of SA at that time. 
Concretely, assume that $\Tmst^*$ is a minimum spanning tree and $\Tmst'$ is the solution of the algorithm at time~$t_{end}$. 
Assume $w(\Tmst)$ is the total weight of edges in the tree~$\Tmst$. 
Using Lemma~\ref{lem:bijection}, with probability~$1-\delta$, we have
$w_{\Tmst'} < (1+\kappa) w_{\Tmst^*}(k)$ for each $k\in [1..n-1]$. Thus, we have
\[w(\Tmst')=\sum_{i=1}^{n-1} w_{\Tmst'}(i) < \sum_{i=1}^{n-1}w_{\Tmst^*}(i) (1+\kappa) = (1+\kappa)w(\Tmst^*).\]

To complete the proof, we only have to find the time $\tend$ from when the temperature is less than $\wmin/a$, so after that, no edges are included anymore via Lemma~\ref{lem:heavy-edges}. Then $\tend$ satisfies
\[T_0(1-1/\ell)^{\tend} = \frac \wmin a.\]
Then
\[\tend = \log_{1-1/\ell}\left((\wmin/a)/T_0\right)=\frac{\ln (\wmin/(aT_0)) }{\ln (1-1/\ell)}.\]
Using the inequality $1-x/2 \ge e^{-x}$  for $0 \le x \le 1$ with $x=2/\ell$, we can bound $\tend$ from above by
\[\tend \le  \frac{\ln (\wmin/(aT_0)) }{-2/\ell} = (\ell/2) \ln \left(\frac{a T_0}{\wmin}\right). \qedhere\]
\end{proof}

The formula for $\kappa$, which we obtained in Theorem~\ref{theo:sa-approximation}, holds for all $\gamma>1$. In the following lemma, we suggest a value for $\gamma$, leading to the smallest value for $1+\kappa$. With the help of that, we give also some bounds on $1+\kappa$ considering different cases for $\ell$.

\begin{lemma} \label{lem:choice-of-parameters}
Let $\kappa$ be defined as in Theorem~\ref{theo:sa-approximation} and $\tbase\coloneqq\timebase$. Then the minimum value of $\kappa$ is achieved by setting $\gamma=\exp\left(\W\left(\frac{\ell-1}{\tbase}\right)\right)$, where $\W$ is the \lambert function.
Moreover, if $\ell < e\tbase+1 $, $1+\kappa \ge e^{(1/e)-1} a$. Otherwise, if $\ell \ge e\tbase+1$,
    \[1+\kappa \le a\frac{\exp\left( \left(\ln \left( \frac{\ell-1}{\tbase} \right)\right)^{\frac{e}{e-1}\ln^{-1} \left( \frac{\ell-1}{\tbase} \right)-1} \right)}{\ln \left( \frac{\ell-1}{\tbase} \right)-\ln \ln \left( \frac{\ell-1}{\tbase} \right)} .\]
   For $\ell=\omega(\tbase)$, the last fraction is  $(1+o(1))\frac{a}{\ln \left( \ell-1\right) - \ln \left( \tbase \right)} $. \\
\end{lemma}
\begin{proof}
From the definition of $\kappa$ in Theorem~\ref{theo:sa-approximation}, for $\gamma>1$, we have
    \begin{align} \label{eq:eps-with-b}
    1+\kappa=a\frac{e^{\gamma/b}}{\ln \gamma},
\end{align}
    where $b\coloneqq \frac{\ell-1}{\tbase}$.

Let $f(x)=e^{x/b}/\ln x$ for $x>1$. Then its derivative is $f'(x)=\frac{e^{x/b}}{b\ln\left(x\right)}-\frac{e^{x/b}}{x\ln^2\left(x\right)}$.
For~$x>1$, we have the only root $x=e^{\W(b)}$, where $\W$ is the \lambert function. Therefore, Equation~\eqref{eq:eps-with-b} with $\gamma=e^{\W(b)}$ gives us the minimum value for $(1+\kappa)$ and equals
\begin{align} \label{eq:eps-with-lambert}
    a \frac{e^{e^{\W(b)}/b}}{\W(b)}.
\end{align}

Now, we aim at finding some bounds on $1+\kappa$. We analyze Equation~\eqref{eq:eps-with-lambert} for two cases of $b$.

For $b \ge e$, using the inequality
        \[\ln b - \ln \ln b + \frac{\ln \ln b}{2\ln b} \le  \W(b) \le \ln b - \ln \ln b + \frac{e}{e-1}\frac{\ln \ln b}{\ln b},\]
        from~\cite{hoorfar2008inequalities}, we get
     \begin{align*}
     a \frac{e^{e^{\W(b)}/b}}{\W(b)} & \le  a \frac{\exp\left(b^{-1}e^{\ln (b)} e^{-\ln \ln b} e^{\frac{e}{e-1}\frac{\ln \ln b}{\ln b}}\right)}{\ln (b) - \ln \ln (b)}   \\
     & = a \frac{\exp\left(e^{-\ln \ln b} e^{\frac{e}{e-1}\frac{\ln \ln b}{\ln b}}\right)}{\ln (b) - \ln \ln (b)} \\
     & = a \frac{\exp\left( (\ln b)^{-1+\frac{e}{(e-1)\ln b}} \right)}{\ln (b) - \ln \ln (b)}.
     \end{align*}
     
     For $b=\omega(1)$, the last expression equals $\frac{a(1+o(1))}{\ln b - \ln \ln b}=(1+o(1))\frac{a}{\ln b}$ since 
     \[(\ln b)^{-1+\frac{e}{(e-1)\ln b}}=\frac{e^{\frac{e\ln \ln b}{(e-1)\ln b}}}{\ln b}\le \frac{1}{\ln b}=o(1).\]
 
 Regarding the case~$b < e $, using the definition $\W(x)e^{\W(x)}=x$, we have $e^{\W(x)}=\frac{x}{\W(x)}$. By applying these inequalities on Equation~\eqref{eq:eps-with-lambert}, we obtain
  \[ a \frac{e^{e^{\W(b)}/b}}{\W(b)} = a\frac{e^{\left(\frac{b}{bW(b)}\right)}}{W(b)} = a \frac{e^{1/W(b)}}{W(b)}. \]
  From the definition again, we have $\W(b)e^{\W(b)}=b$. Since for $x\ge0$, we have $e^x\ge 1$, we can conclude $\W(b)\le b$, resulting in $\W(b)< e$. Thus the last expression can be bounded from below by
  \[a\frac{e^{1/e}}{e}=e^{(1/e)-1}a. \qedhere\]
\end{proof}

Finally, we give the proofs of the two 
main theorems in this paper.

\begin{proofof}{Theorem~\ref{theo:eps-from-ell}}
Using Theorem~\ref{theo:sa-approximation}, we have
\[1+\kappa=\frac{a\exp\left(\gamma \frac{\tbase}{\ell-1}\right)}{\ln \gamma}.\]
By setting $a=\minimuma$ and using the upper bound on $(1+\kappa)$  obtained in Lemma~\ref{lem:choice-of-parameters} for $\ell=\omega(\tbase)=\omega(mn\ln (m/\delta))$, we get
\begin{align*}
    1+\kappa &\le (1+o(1)) \frac{\minimuma}{\ln (\ell-1) - \ln (\timebase)}  \\
    & = (1+o(1)) \cdot (1+o(1))\frac{\ln((\ell-1)/\delta)}{\ln (\ell) - \ln (mn\ln (m/\delta))} \\
    & \le (1+o(1)) \frac{\ln(\ell/\delta)}{\ln (\ell) - \ln (mn\ln (m/\delta))}. \qedhere
\end{align*}
\end{proofof}

In Theorem~\ref{theo:eps-from-ell}, we only consider the case $\ell=\omega(\tbase)$ since the other cases for $\ell$ cannot lead to constant approximation ratios and therefore are not interesting to study. More precisely, let us assume $\ell=\omega(1)$. In the case that $\ell<e\tbase+1$, we have the lower bound $\Omega(\minimuma)=\omega(1)$ on $1+\kappa$ from Lemma~\ref{lem:choice-of-parameters}. Regarding the case that $\ell\ge e\tbase+1 $ and $\ell=O(\tbase)$, it can be proved that
$1+\kappa = \Omega(a)=\omega(1)$,
since $\ell/\tbase=O(1)$ makes all terms constant except~$a$ in Equation~\eqref{eq:eps-with-lambert}. Then again for $a\ge \minimuma$ and $\ell=\omega(1)$, the approximation ratio is $\omega(1)$.

Now, we give the proof of Theorem~\ref{theo:approx-depend-ell}.
\begin{proofof}{Theorem~\ref{theo:approx-depend-ell}}
Let $\ell=\left(mn\ln (m/\delta)\right)^{1+ 1/\eps}$. Via Theorem~\ref{theo:eps-from-ell}, we have
\begin{align*}
    1+\kappa &\le (1+o(1))\frac{\ln (\ell/\delta)}{\ln \left(\frac{\ell}{mn\ln (m/\delta)}\right)} \\
    & = (1+o(1))\frac{(1+1/\eps)\ln \left(mn \ln (m/\delta)\right) + \ln (1/\delta)}{(1/\eps) \ln \left(mn \ln (m/\delta)\right)} \\
    & = (1+o(1))\left(  \frac{1+1/\eps}{1/\eps} + \frac{\ln (1/\delta)}{(1/\eps) \ln (mn \ln (m/\delta))} \right) \\
    & \le (1+o(1))\left(1+\frac{\ln (1/\delta)}{\ln (mn \ln (m/\delta))}\right)\left(  1+ \eps \right).
\end{align*}

For $\delta^{-1}=o(mn\ln n)$, the last expression can be bounded from above by
     $(1+o(1))\left(  1+ \eps \right)$.
\end{proofof}

A more straightforward result of Theorem~\ref{theo:approx-depend-ell} is stated in Corollary~\ref{cor:final-corollary}. In this corollary, we are aiming at expressing an asymptotic time for the algorithm to find the 
approximation, and we assume that $\eps$ is constant.

\begin{proofof}{Corollary~\ref{cor:final-corollary}}
Using Theorem~\ref{theo:approx-depend-ell}, we
will first prove the result for 
an approximation ratio of $(1+o(1))(1+\eps')$ for some 
constant~$\eps'>0$ and then bound this by
a ratio of 
at most $1+\eps$ such that 
$(1+o(1))(1+\eps')\le 1+\eps$ 
for $n$ large enough. 

Note that $\ell=(mn\ln (m/\delta))^{1+1/\eps'}$ and $\delta=\omega(1/(mn\ln n))$ and invoke Theorem~\ref{theo:approx-depend-ell}.
The asymptotic bound on $T^*$ is obtained in the following way: we note that 
$\ln(n/\delta)=O(\ln n)$ since $1/\delta=n^{O(1)}$ by assumption and $m\le n^2$. 
Since 
$\eps'>0$ is constant, we have 
$\ln(\ell)=O((1+1/\eps')\ln(mn \ln(m/\delta)))=O((1+1/\eps')\ln n)))=O(\ln n)$. Moreover, $\ell=O((mn\ln(n/\delta))^{1+\eps'})=
O((mn\ln(n))^{1+\eps'})$. Putting this together, we have 
\begin{align*}
T^* = 
O((mn\ln(n))^{1+1/\eps'} (\ln \ln n+\ln(T_0/\wmin))).
\end{align*}

We have that $1/\eps'=1/\eps+o(1)$ 
since $\eps$ and $\eps'$ are constants. 
Hence, we obtain 
the statement of the corollary.
\end{proofof}

\section{$(1+\eps)$-separated weights}
\label{sec:sa-eps-separated}
In this section, we revisit the case that the weights $w_1,\dots,w_m$
are 
\emph{$(1+\eps)$-separated}, \ie, there is a constant $\eps>0$ such 
that $w_{j}\ge (1+\eps)w_i$ if $w_j>w_i$ for all $i,j\in\{1,\dots,n\}$. As mentioned in the introduction in Theorem~\ref{thm:wegener},
Wegener proves that SA 
 with high probability finds an MST for any  
instance with $(1+\eps)$-separated weights if 
$\wmax\le 2^m$. More precisely, the proof of his  
theorem considers a time span of $O(m^{8+8/\eps})$ steps and shows that 
SA constructs an MST within this time span with 
probability $1-O(1/m)$. 

In the following, we improve this result in two ways. As acknowledged by Wegener himself, 
he did not optimize the parameters in the final bound on the runtime. Therefore, 
we can 
give an improved time bound of $O((mn\ln(n))^{1+1/\eps+o(1)}(\ln \ln n+\ln(T_0/\wmin)))$, see Theorem~\ref{theo:sa-eps-separated} for the precise, 
more general result. Moreover, we replace the assumption on the largest edge weight by the parameter~$w_{\max}$.
Essentially, we have done all work necessary to show the following theorem already in the previous section, where we proved 
an approximation result. Now, the $(1+\eps)$-separation implies that indeed an optimal solution is found 
with high probability.

\begin{theorem}
\label{theo:sa-eps-separated}
Let $\delta=\omega(1/(mn \ln (m)))$ and $\delta < 1$, $\eps>0$ be a constant.
Consider a run of SA with multiplicative cooling schedule with $\beta = 1-1/\ell$ for $\ell = (mn\ln (m/\delta))^{1+ 1/\eps+o(1)}$ and $T_0 \ge \wmax$ on an instance of the MST problem with $(1+\eps)$-separated weights. 
With probability at least~$1-\delta$, at all times $t \ge T^* \coloneqq (\ell/2) \ln \left(\frac{\minimuma T_0}{\wmin}\right)$ the current solution is optimal. Moreover, 
\[T^*=O((mn\ln(n))^{1+1/\eps+o(1)}(\ln \ln n+\ln(T_0/\wmin))).\]
\end{theorem}

\begin{proof} We first prove the result for $(1+o(1))(1+\eps')$-separated weights for some constant~$\eps'>0$. Then we prove the result for $(1+\eps)$-separated weights such that $(1+o(1))(1+\eps')\le (1+\eps)$ for~$n$ large enough.

Using Lemma~\ref{lem:bijection}, with probability~$1-\delta$, we have
$w_{\Tmst^*}(k) \le w_{\Tmst'}(k) < (1+\kappa) w_{\Tmst^*}(k)$ for each $k\in [1..n-1]$. 
The $(1+\kappa)$-separated graphs do not have 
edge weight between $w_{\Tmst^*}(k)$ and $(1+\kappa) w_{\Tmst^*}(k)$ except $w_{\Tmst^*}(k)$. Therefore, the algorithm finds an optimal solution.

We need to bound $1+\kappa$ using the assumptions in the statement. By setting $a=\minimuma$ and using Lemma~\ref{lem:choice-of-parameters} for $\ell = \left(mn\ln (m/\delta)\right)^{1+1/{\eps'}}$, we bound $1+\kappa$ from above by $(1+o(1))(1+\eps')$ similarly to the proof of Theorem~\ref{theo:approx-depend-ell}. Since $\eps$ and $\eps'$ are constants and we have $1/\eps'=1/\eps+o(1)$, we obtain the claim for  $(1+\eps)$-separated weights.

Regarding $T^*$, since 
$\eps>0$ is constant, we have 
$\ln(\ell)=O((1+1/\eps+o(1))\ln(mn\ln (m/\delta)))=O((1+1/\eps+o(1))\ln n)))=O(\ln n)$. Moreover, $\ell=O((mn\ln(n/\delta))^{1+\eps+o(1)})=
O((mn\ln(n))^{1+\eps+o(1)})$. Putting this together, we have 
\[
T^* = 
O((mn\ln(n))^{1+1/\eps+o(1)} (\ln \ln n+\ln(T_0/\wmin))). \qedhere\] 
\end{proof}

\section{Conclusions}
We have shown that simulated annealing is a polynomial-time approximation scheme for the minimum spanning tree problem, thereby proving a conjecture by Wegener \citep{Wegener05}. Our analyses use state-of-the-art methods and have led to improved results in the case of $(1+\epsilon)$-separated weights, where simulated annealing yields an optimal solution with high probability.
Our main result is one of the rare examples where simple randomized search heuristics, with a straightforward representation and objective function, serve as polynomial-time approximation scheme. 

Since the runtime analysis of simulated annealing is still underrepresented in the theory of randomized search heuristics, our understanding of its  working principles is still limited. In particular, we do not have a clear characterization of the fitness landscapes in which its non-elitism, along with a cooling schedule, is more efficient than global search. The study of the Metropolis Algorithm for the DLB problem
in~\cite{WangZD21} 
and our analysis on the minimum spanning tree problem might indicate that  landscapes with many, but easy to leave local optima are beneficial; however, more research is needed to support this conjecture. 

\section*{Acknowledgement}
 This work was supported by a public grant as part of the
Investissements d'avenir project, reference ANR-11-LABX-0056-LMH,
LabEx LMH, and a grant by the Independent Research Fund Denmark (DFF-FNU  8021-00260B).

\bibliographystyle{mynatbib_english}

\bibliography{alles_ea_master,references,ich_master}

\begin{thebibliography}{23}
\expandafter\ifx\csname natexlab\endcsname\relax\def\natexlab#1{#1}\fi
\expandafter\ifx\csname url\endcsname\relax
  \def\url#1{\texttt{#1}}\fi
\expandafter\ifx\csname urlprefix\endcsname\relax\def\urlprefix{URL }\fi

\bibitem[{Corus, Oliveto and Yazdani(2019)Corus, Oliveto, and
  Yazdani}]{CorusOYAIJ19}
Corus, Dogan, Oliveto, Pietro~S., and Yazdani, Donya (2019).
\newblock Artificial immune systems can find arbitrarily good approximations
  for the np-hard number partitioning problem.
\newblock \emph{Artificial Intelligence}, \textbf{274}, 180--196.

\bibitem[{Doerr(2021)}]{Doerr21tcsUB}
Doerr, Benjamin (2021).
\newblock Exponential upper bounds for the runtime of randomized search
  heuristics.
\newblock \emph{Theoretical Computer Science}, \textbf{851}, 24--38.

\bibitem[{Doerr and Goldberg(2013)}]{DoerrG13algo}
Doerr, Benjamin and Goldberg, Leslie~A. (2013).
\newblock Adaptive drift analysis.
\newblock \emph{Algorithmica}, \textbf{65}, 224--250.

\bibitem[{Doerr, Johannsen and Winzen(2012)Doerr, Johannsen, and
  Winzen}]{DoerrJW12algo}
Doerr, Benjamin, Johannsen, Daniel, and Winzen, Carola (2012).
\newblock Multiplicative drift analysis.
\newblock \emph{Algorithmica}, \textbf{64}, 673--697.

\bibitem[{Droste, Jansen and Wegener(2000)Droste, Jansen, and
  Wegener}]{DrosteJW00}
Droste, Stefan, Jansen, Thomas, and Wegener, Ingo (2000).
\newblock Dynamic parameter control in simple evolutionary algorithms.
\newblock In \emph{Foundations of Genetic Algorithms, FOGA 2000}, 275--294.
  Morgan Kaufmann.

\bibitem[{Franzin and St{\"{u}}tzle(2019)}]{FranzinS19}
Franzin, Alberto and St{\"{u}}tzle, Thomas (2019).
\newblock Revisiting simulated annealing: {A} component-based analysis.
\newblock \emph{Computers and Operations Research}, \textbf{104}, 191--206.

\bibitem[{Giel and Wegener(2003)}]{GielW03}
Giel, Oliver and Wegener, Ingo (2003).
\newblock Evolutionary algorithms and the maximum matching problem.
\newblock In \emph{Symposium on Theoretical Aspects of Computer Science, STACS
  2003}, 415--426. Springer.

\bibitem[{Hoorfar and Hassani(2008)}]{hoorfar2008inequalities}
Hoorfar, Abdolhossein and Hassani, Mehdi (2008).
\newblock Inequalities on the lambert w function and hyperpower function.
\newblock \emph{Journal of Inequalities in Pure \& Applied Mathematics},
  \textbf{9}, 5--9.

\bibitem[{Horoba(2010)}]{HorobaECJ10}
Horoba, Christian (2010).
\newblock Exploring the runtime of an evolutionary algorithm for the
  multi-objective shortest path problem.
\newblock \emph{Evolutionary Computation}, \textbf{18}(3), 357--381.

\bibitem[{Jansen(2011)}]{Jansen11}
Jansen, Thomas (2011).
\newblock Simulated annealing.
\newblock In Auger, Anne and Doerr, Benjamin (eds.), \emph{Theory of Randomized
  Search Heuristics}, 171--195. World Scientific Publishing.

\bibitem[{Jansen and Wegener(2007)}]{JansenW07}
Jansen, Thomas and Wegener, Ingo (2007).
\newblock A comparison of simulated annealing with a simple evolutionary
  algorithm on pseudo-{B}oolean functions of unitation.
\newblock \emph{Theoretical Computer Science}, \textbf{386}, 73--93.

\bibitem[{Jerrum and Sorkin(1998)}]{JerrumS98}
Jerrum, Mark and Sorkin, Gregory~B. (1998).
\newblock The metropolis algorithm for graph bisection.
\newblock \emph{Discrete Applied Mathematics}, \textbf{82}, 155--175.

\bibitem[{Kirkpatrick, Gelatt~Jr and Vecchi(1983)Kirkpatrick, Gelatt~Jr, and
  Vecchi}]{KirkpatrickGV83}
Kirkpatrick, Scott, Gelatt~Jr, C~Daniel, and Vecchi, Mario~P (1983).
\newblock Optimization by simulated annealing.
\newblock \emph{Science}, \textbf{220}, 671--680.

\bibitem[{Lehre and Witt(2021)}]{LehreW21}
Lehre, Per~Kristian and Witt, Carsten (2021).
\newblock Tail bounds on hitting times of randomized search heuristics using
  variable drift analysis.
\newblock \emph{Combinatorics, Probability \& Computing}, \textbf{30}(4),
  550--569.

\bibitem[{Lissovoi, Oliveto and Warwicker(2019)Lissovoi, Oliveto, and
  Warwicker}]{LissovoiOW19}
Lissovoi, Andrei, Oliveto, Pietro~S., and Warwicker, John~Alasdair (2019).
\newblock On the time complexity of algorithm selection hyper-heuristics for
  multimodal optimisation.
\newblock In \emph{Conference on Artificial Intelligence, {AAAI} 2019},
  2322--2329. {AAAI} Press.

\bibitem[{Neumann and Sutton(2020)}]{NeumannSuttonFPTBookChapter2020}
Neumann, Frank and Sutton, Andrew~M. (2020).
\newblock Parameterized complexity analysis of randomized search heuristics.
\newblock In Doerr, Benjamin and Neumann, Frank (eds.), \emph{Theory of
  Evolutionary Computation -- Recent Developments in Discrete Optimization},
  213--248. Springer.

\bibitem[{Neumann and Wegener(2007)}]{NeumannW07}
Neumann, Frank and Wegener, Ingo (2007).
\newblock Randomized local search, evolutionary algorithms, and the minimum
  spanning tree problem.
\newblock \emph{Theoretical Computer Science}, \textbf{378}, 32--40.

\bibitem[{Neumann and Witt(2010)}]{NeumannW10}
Neumann, Frank and Witt, Carsten (2010).
\newblock \emph{Bioinspired Computation in Combinatorial Optimization --
  Algorithms and Their Computational Complexity}.
\newblock Springer.

\bibitem[{Oliveto et~al.(2018)Oliveto, Paix{\~{a}}o, Heredia, Sudholt, and
  Trubenov{\'{a}}}]{OlivetoPHST18}
Oliveto, Pietro~S., Paix{\~{a}}o, Tiago, Heredia, Jorge~P{\'{e}}rez, Sudholt,
  Dirk, and Trubenov{\'{a}}, Barbora (2018).
\newblock How to escape local optima in black box optimisation: when
  non-elitism outperforms elitism.
\newblock \emph{Algorithmica}, \textbf{80}, 1604--1633.

\bibitem[{Sasaki and Hajek(1988)}]{SasakiH88}
Sasaki, Galen~H. and Hajek, Bruce (1988).
\newblock The time complexity of maximum matching by simulated annealing.
\newblock \emph{Journal~of the ACM}, \textbf{35}, 387--403.

\bibitem[{Wang, Zheng and Doerr(2021)Wang, Zheng, and Doerr}]{WangZD21}
Wang, Shouda, Zheng, Weijie, and Doerr, Benjamin (2021).
\newblock Choosing the right algorithm with hints from complexity theory.
\newblock In \emph{International Joint Conference on Artificial Intelligence,
  {IJCAI} 2021}, 1697--1703. ijcai.org.

\bibitem[{Wegener(2005)}]{Wegener05}
Wegener, Ingo (2005).
\newblock Simulated annealing beats {M}etropolis in combinatorial optimization.
\newblock In \emph{Automata, Languages and Programming, {ICALP} 2005},
  589--601. Springer.

\bibitem[{Witt(2005)}]{WittSTACS05}
Witt, Carsten (2005).
\newblock Worst-case and average-case approximations by simple randomized
  search heuristics.
\newblock In Diekert, Volker and Durand, Bruno (eds.), \emph{Proc.\ of
  STACS~2005}, vol. 3404 of \emph{Lecture Notes in Computer Science}, 44--56.
  Springer.

\end{thebibliography}

\appendix

}

\end{document}